\documentclass{jps-cp}
\usepackage{algorithm}
\usepackage{algorithmic}
\usepackage{wrapfig}
\usepackage{graphicx}
\usepackage{mathtools} 
\usepackage{amssymb}
\usepackage{amsmath}
\usepackage{amsthm} 
\usepackage{color}
\usepackage{bm}

\newtheorem{theorem}{Theorem}

\newtheorem{definition}{Definition}
\usepackage{tcolorbox}
\usepackage{url}

\DeclareMathOperator{\diag}{diag}
\DeclareMathOperator{\trace}{Tr}

\DeclareMathOperator{\MLP}{MLP}

\DeclareMathOperator{\Linear}{Linear}

\DeclareMathOperator{\ReLU}{ReLU}

\newcommand{\bmb}{\bm{b}}

\newcommand{\bmg}{\bm{g}}
\newcommand{\bmh}{\bm{h}}

\newcommand{\bmm}{\bm{m}}

\newcommand{\bmv}{\bm{v}}
\newcommand{\bmw}{\bm{w}}
\newcommand{\bmx}{\bm{x}}

\newcommand{\bmz}{\bm{z}}

\newcommand{\rmc}{\mathrm{c}}
\newcommand{\rmd}{\mathrm{d}}

\newcommand{\bmzero}{\bm{0}}

\newcommand{\bmtheta}{\bm{\theta}}

\newcommand{\bmmu}{\bm{\mu}}

\newcommand{\bmpi}{\bm{\pi}}

\newcommand{\bmsigma}{\bm{\sigma}}

\newcommand{\bmphi}{\bm{\phi}}

\newcommand{\sfA}{\mathsf{A}}

\newcommand{\sfC}{\mathsf{C}}

\newcommand{\sfH}{\mathsf{H}}
\newcommand{\sfI}{\mathsf{I}}

\newcommand{\sfQ}{\mathsf{Q}}

\newcommand{\sfW}{\mathsf{W}}

\newcommand{\sfZ}{\mathsf{Z}}

\newcommand{\sfLambda}{\mathsf{\Lambda}}

\newcommand{\sfSigma}{\mathsf{\Sigma}}

\newcommand{\calD}{\mathcal{D}}

\newcommand{\calN}{\mathcal{N}}

\usepackage{txfonts} 

\title{Decentralized Collaborative Learning Framework with External Privacy Leakage Analysis}

\author{
Tsuyoshi \textsc{Id\'e}$^1$, Dzung T. \textsc{Phan}$^1$, and Rudy \textsc{Raymond}$^2$ 
}

\inst{
$^1$IBM Thomas J. Watson Research Center. \\
$^2$ Global Technology and Applied Research, J.P. Morgan Chase \& Co.
}

\email{\{tide,phandu\}@us.ibm.com,  raymond.putra@jpmchase.com} 

\recdate{February 4, 2024}

\abst{
This paper presents two methodological advancements in decentralized multi-task learning under privacy constraints, aiming to pave the way for future developments in next-generation Blockchain platforms. First, we expand the existing framework for collaborative dictionary learning (CollabDict), which has previously been limited to Gaussian mixture models, by incorporating deep variational autoencoders (VAEs) into the framework, with a particular focus on anomaly detection. We demonstrate that the VAE-based anomaly score function shares the same mathematical structure as the non-deep model, and provide comprehensive qualitative comparison. Second, considering  the widespread use of ``pre-trained models,'' we provide a mathematical analysis on data privacy leakage when models trained with CollabDict are shared externally. We show that the CollabDict approach, when applied to Gaussian mixtures, adheres to a R\'enyi differential privacy criterion. Additionally, we propose a practical metric for monitoring internal privacy breaches during the learning process. 
}

\kword{
Blockchain, collaborative learning, multi-task learning, anomaly detection, variational autoencoder}

\begin{document}
\maketitle

\section{Introduction}

The design principles of Blockchain, including decentralization, security, and transparency in transaction management, have had a profound impact on business and industrial applications. While Blockchain technology was initially proposed as a self-governing platform for currency exchange~\cite{nakamoto2008bitcoin}, it was later expanded to include general business transactions in the form of smart contracts~\cite{christidis2016blockchains}. However, in many application scenarios, such as product traceability systems, Blockchain has primarily been used as an immutable data storage system where stored data is never transformed in any way. We believe that the true value of Blockchain lies in its potential for value co-creation through knowledge sharing. For example, if car dealers collaboratively store pre-owned car repair records, they may wish to have an anomaly detection (or diagnosis) model trained on the collected data with a machine learning algorithm. This involves the process of sharing and transforming data. We envision that the next generation of Blockchain will seamlessly integrate machine learning algorithms, enabling collaborative learning functionalities.

From a machine learning perspective, collaborative learning can be formalized as decentralized multi-task learning under a data privacy constraint. Figure~\ref{fig:CbMnetwork.pdf} illustrates the problem setting. The network consists of $S$ participant members ($S>2$) connected to each other through a sparse peer-to-peer (P2P) network. Each participant, indexed with $s$, has its own dataset $\calD^s=\{\bmx^{s(1)}, \ldots, \bmx^{s(N^s)}\}$, where $\bmx^{s(n)} \in \calD^s$ represents the $n$-th sample of the $s$-th participant and $N^s$ is the number of samples in $\calD^s$. The participants' objective is to develop their own machine learning models. Since this involves learning $S$ distinct models simultaneously, it is referred to as \textit{multi-task learning} in the machine learning literature.

\begin{figure}[t]
\begin{center}
\includegraphics[trim={7cm 0.5cm 5cm 1cm},clip,width=6cm]{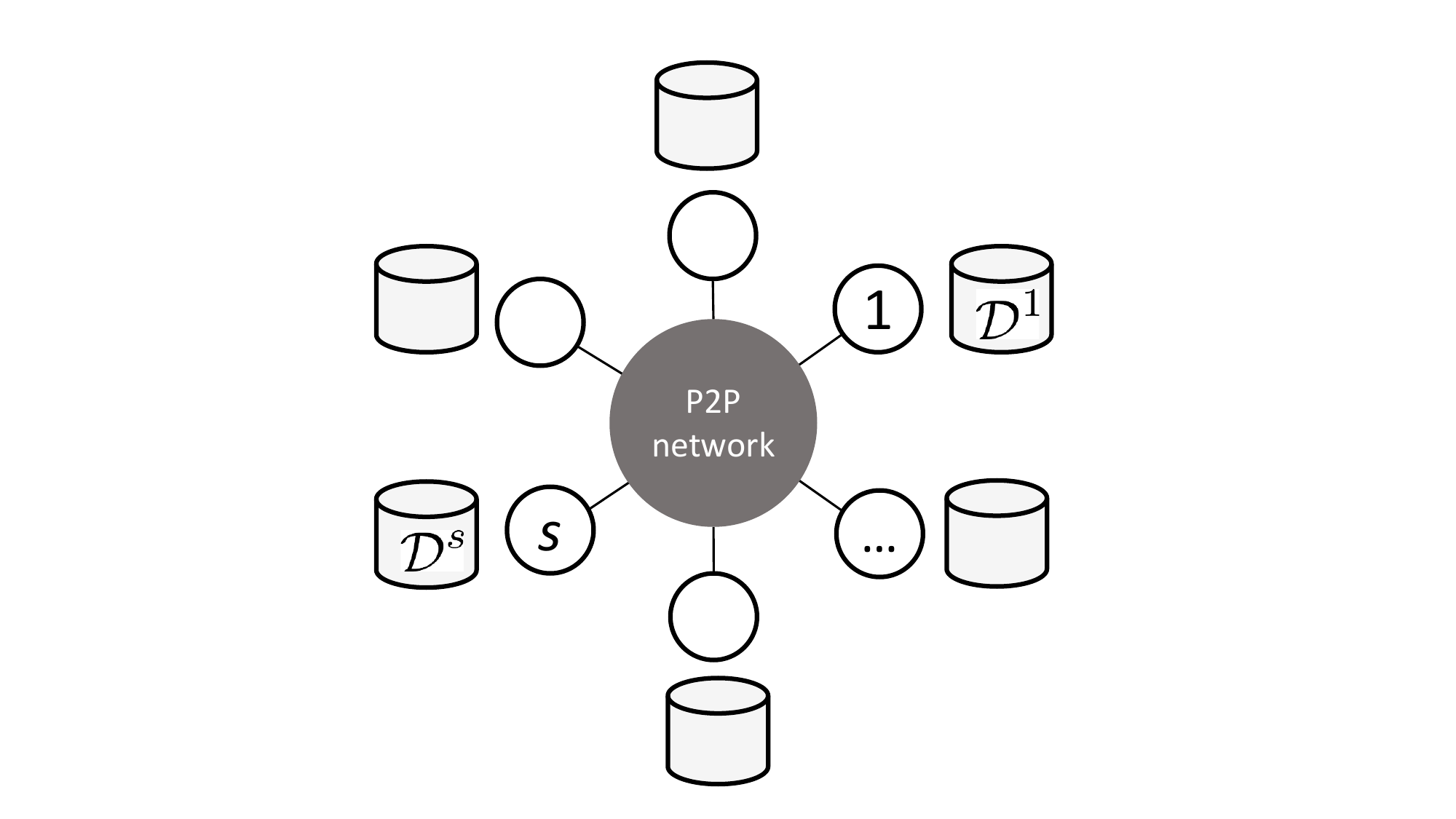}
\hspace{1cm}
\includegraphics[width=4cm]{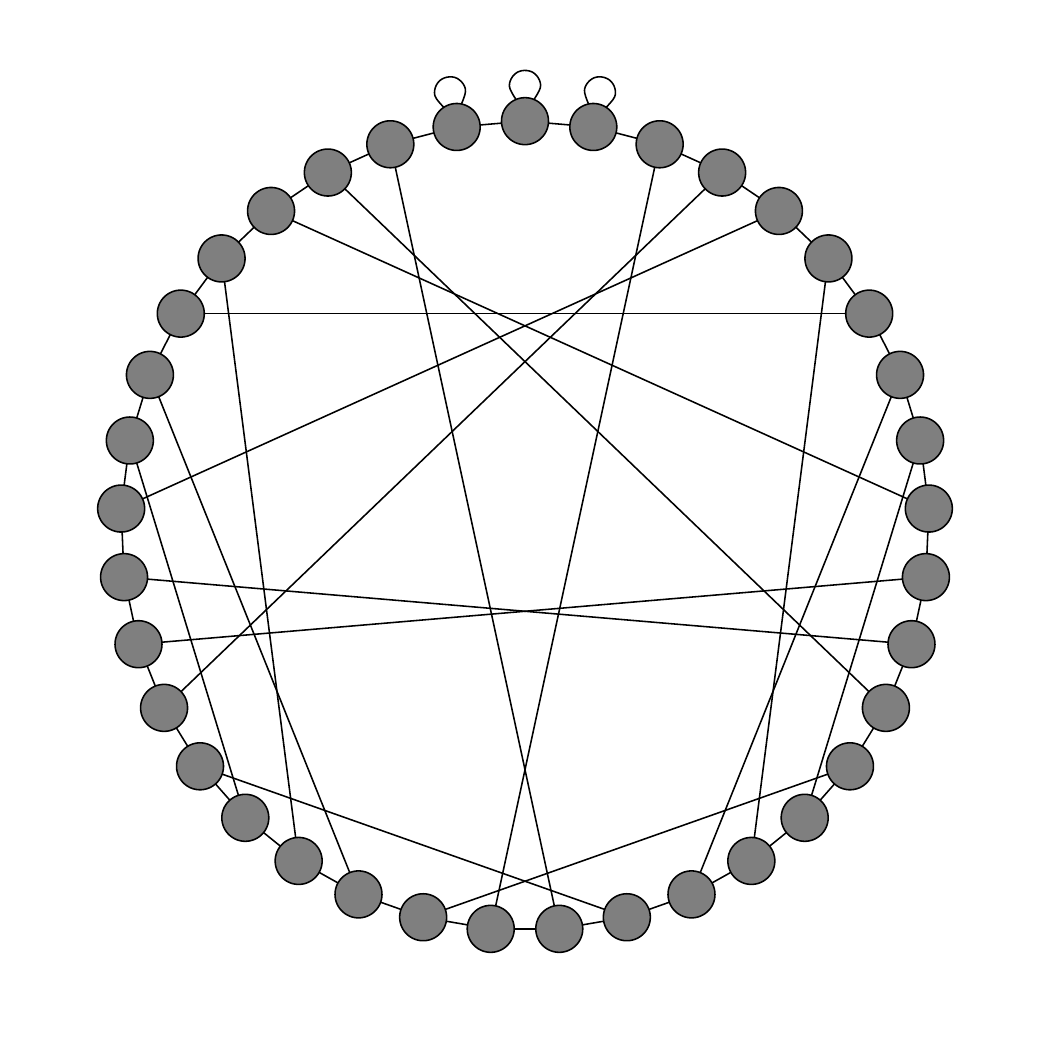}
\end{center}
\caption{Illustration of the problem setting of multi-task learning under the privacy and decentralization constraints. \textit{Left}: Each network participant ($s=1,\ldots, S$) holds its own dataset $\calD^s$ privately and builds its own machine learning model with the hope that other participants' data would help improve the model. \textit{Right}: An example of a peer-to-peer (P2P) communication network. A 3-regular expander graph called the cycle with inverse chord is shown for $S=31$ participants. }
\label{fig:CbMnetwork.pdf}
\end{figure}

While multi-task learning under decentralized and privacy-preservation constraints is an interesting extension of the traditional machine learning paradigm, meeting both constraints is generally challenging. This is because these constraints work in opposite directions: Decentralization implies ``sharing with anyone,'' whereas privacy implies ``\textit{not} sharing with anyone.'' CollabDict (collaborative dictionary learning)~\cite{ide2018collaborative,ide2019efficient} is the first framework that successfully addresses both. However, there are two major limitations. \textit{First}, the framework was specifically derived for the sparse Gaussian graphical model (GGM) mixtures, and it is not clear whether it is applicable to more general settings, such as those using deep learning. \textit{Second}, while a previous study~\cite{ide2021decentralized} analyzes the risk of privacy breach within the network in the training phase (i.e., internal privacy leakage), little is known about external privacy leakage when the trained model is shared with a third party as a pre-trained model.

In this paper, we introduce a novel collaborative learning framework, with a particular focus on unsupervised anomaly detection, based on a multi-task extension of the variational autoencoder (VAE)~\cite{kingma2013auto}. Unlike traditional autoencoders (AEs), VAEs are capable of providing a predictive probability distribution of the observed variables. This allows us to quantify the range of normalcy in a principled way, making VAEs more useful than AEs in practice. While VAEs have been extensively used in anomaly detection~\cite{an2015variational,xu2018unsupervised,chen2020unsupervised}, the multi-task extension under privacy and decentralized constraints has been considered challenging. We show that the proposed multi-task VAE can naturally fit the CollabDict framework. Interestingly, we will point out that the new VAE-based approach shares the same mixture representation for the anomaly score function, demonstrating the general applicability of the original concept of CollabDict. Encouraged by this fact, we provide a theoretical analysis of external privacy leakage under the GGM setting with the aid of the R\'enyi differential privacy theory~\cite{mironov2017renyi}. To the best of our knowledge, this is the first work that provides a theoretical guarantee on external privacy leakage under the decentralized setting.

To summarize, our contributions are twofold:
\begin{itemize}
    \item We propose a novel VAE-based multi-task anomaly detection framework under decentralized and data privacy constraints. 
    \item We provide the first mathematical guarantee of external data leakage when a model trained on CollabDict is used externally by a third party. 
\end{itemize}

\section{Related Work}\label{sec:related}

Most of the applications of Blockchain in domains involving real-valued noisy data are related to product traceability~\cite{tse2017blockchain, lu2017adaptable, toyoda2017novel}. There are few studies in the literature on collaborative learning. Recently, M{\"u}nsing et al.~proposed a Blockchain-based optimization framework for efficiently scheduling a micropower grid~\cite{munsing2017blockchains}. This framework shares the concept of separating local and global variables with ours, but it is not applicable to multi-task anomaly detection and does not discuss privacy preservation. Other partially relevant studies include~\cite{xie2017privacy}, which discusses privacy-preserving multi-task learning (MTL) in a distributed setting. However, it lacks the context of Blockchain, especially consensus building, and its mathematical treatment of privacy seems incomplete. Another recent paper~\cite{ide2018collaborative} proposes an MTL-based fault detection framework, but it does not cover privacy preservation and consensus building.

Data privacy protection in the distributed learning setting has been actively studied in the field of federated learning. Chen et al.~\cite{chen2021ds2pm} studied privacy protection in federated learning and proposed an approach based on homomorphic encryption. While the proposed framework allows decentralized secret data exchange, handling the diversity of network participants and therefore multi-task learning is not within the main scope. Masurkar et al.~\cite{masurkar2023using} proposed another decentralized, privacy-preserving federated learning framework, but the aspects of multi-task learning and anomaly detection are not discussed. Mahmood et al.~\cite{mahmood2022blockchain} address the risk of having malicious participants in privacy-preserving federated learning.

The work by Zhao et al.~\cite{zhao2019multi} may be closest to ours, which proposed a deep learning framework in the multi-task and distributed learning setting. The idea of global and local updates, which is shared by the CollabDict framework, is included there. However, it is only in the supervised setting, and the communication protocol relies on a central server.

Training deep learning models in a distributed setting is one of the recent popular topics in machine learning. In the context of VAE, for example, recent works include Polato\cite{polato2021federated} and Li et al.~\cite{li2023distvae}, which proposed federated VAE learning for collaborative filtering. Also, Zhao et al.~\cite{zhao2024deep} proposed a secure distributed VAE based on homomorphic encryption. Zhang et al.~\cite{zhang2021federated} and Huang et al.~\cite{huang2022novel} addressed the task of anomaly detection from multivariate time-series data using VAE. These works are either not for multi-task learning or not in the decentralized setting. Very recently, much attention has been paid to multi-task reinforcement learning (MTRL)~\cite{omidshafiei2017deep, zeng2021decentralized}, in which the goal is to learn a policy shared by multiple tasks, but standard anomaly detection tasks are typically not within their scope.

\section{Problem Setting}\label{sec:setting}

We consider a scenario where $S$ parties wish to collaboratively build anomaly detection models through communication on a P2P network. Depending on the communication protocol, the network can be dynamically created, but participants are assumed to know who their peers are during communication. Each network participant, indexed with $s\in \{1,\ldots, S\}$, has its own dataset $\calD^s = \{\bmx^{s(1)}, \ldots, \bmx^{s(N^s)} \}$. Here, $\bmx^{s(n)}$ represents the $n$-th sample of the $s$-th dataset, and $N^s$ is the number of samples in $\calD^s$. It is possible for some participants to have more samples than others. To distinguish between a random variable and its realization, we attach an index specifying an instance. For example, $\bmx^s$ is a random variable representing participant $s$'s data generation mechanism, while $\bmx^{s(n)}$ is a realization (i.e., a numerical vector). We refer to the participant's data generating mechanism as the \textit{system}. We assume that all samples are in $\mathbb{R}^M$ (i.e., $M$-dimensional real-valued vectors) and follow the same data format. For instance, if the 8th dimension of the second participant's data, $x_8^{2(n)}$, is a temperature measurement in Celsius, then $x^{s(n)}_8$ represents temperatures in Celsius for any $s,n$.

We focus on an unsupervised learning scenario in anomaly detection. The observed samples $\{ \bmx^{s(n)} \}$ are assumed to be generally noisy. One standard definition of anomaly scores in such a situation is the logarithmic loss (See, e.g.,~\cite{lee2000information,Yamanishi2000,staniford2002practical,noto2010anomaly}):
\begin{align}\label{ed:anomaly_score_general}
a^s(\bmx^{s,\mathrm{test}}) = - \ln p^s(\bmx^{s,\mathrm{test}}),
\end{align}
which can be also viewed as the negative one-sample likelihood with respect to the predictive distribution. An unusually low likelihood is an indication of anomalousness, yielding a high anomaly score. 
Here, $p^s(\cdot)$ denotes the estimated probability density function (pdf) of the $s$-th system. Unlike straightforward metrics such as the reconstruction error~\cite{Chandola09AnomalySurvey,chapel2014anomaly,ruff2021unifying}, the log loss score can handle prediction uncertainty in a systematic manner. The definition of~\eqref{ed:anomaly_score_general} is supported also from an information-theoretic perspective~\cite{yamanishi2023learning}. 

The task of unsupervised anomaly detection is now a density estimation problem. It is important to note that the pdf depends on $s$. Thus, the entire task is to learn $S$ pdfs simultaneously. In machine learning, this setting is typically referred to as \textit{multi-task learning}. 
A key concept in multi-task learning is task-relatedness. However, under the constraint of data privacy, where participants cannot directly share their local data with others, leveraging other participants' data for improving the models is not straightforward. The key research question here is how selfish participants can collaboratively build the model while maintaining data privacy.

\section{The CollabDict Framework}

CollabDict~\cite{ide2018collaborative,ide2019efficient} is an iterative procedure derived from the maximum a posterior (MAP) principle to perform multi-task learning under privacy and decentralized constraints. The algorithm aims to find model parameters of the pdf $p^s(\cdot)$ by maximizing the log marginalized likelihood. It is referred to as ``dictionary learning'' since it is designed to discover a pool of multiple patterns (i.e., a pattern dictionary) so that diversity over the participants' systems is properly captured. See Subsection~\ref{subsec:anomaly_score_Gaussian_mixture} for further discussion. 

In CollabDict, $p^s(\cdot)$ is generally parameterized as $p^s(\cdot \mid \bmphi^s, \bmtheta)$, where $\bmphi^s$ is the set of local parameters that are privately held within the $s$-th participant and $\bmtheta$ is the set of global parameters that are shared by all participants. Algorithm~\ref{algo:CoDiBlock} shows the overall procedure of CollabDict, where $\{\bmtheta^s\}$ are intermediate variables used to compute $\bmtheta$. At convergence, each participant ends up having $(\bmphi^s, \bmtheta)$ at hand as the final outcome. It consists of three main modules: $\mathtt{local\_update}, \mathtt{consensus}$, and $\mathtt{optimize}$. In these modules, $\mathtt{local\_update}$ and $\mathtt{optimize}$ loosely correspond to the expectation and maximization in the EM (expectation-maximization) algorithm~\cite{Bishop}, respectively. The $\mathtt{consensus}$ module is unique to privacy-preserving decentralized computation and is discussed in the next sections. Following the previous work, we assume that all the participants are \textit{honest but curious}, meaning that they do not lie about their data and computed statistics but they always try to selfishly get as much information as possible from the others.

\begin{algorithm}[bth]
\caption{\texttt{CollabDict} (collaborative dictionary learning)}
\label{algo:CoDiBlock}
\begin{algorithmic}
\STATE Initialize local and global model parameters. 
\STATE Set hyper-parameters to the agreed-upon values.
\REPEAT
\FOR {$s \leftarrow 1, \ldots, S$}
\STATE $\bmphi^s,\bmtheta^s \leftarrow \mathtt{local\_update}(\bmphi^s,  \bmtheta)$
\ENDFOR
\STATE $\bmtheta \leftarrow \mathtt{consensus}(\bmtheta^1, \ldots, \bmtheta^S)$
\FOR {$s \leftarrow 1, \ldots, S$}
\STATE $\bmtheta \leftarrow \mathtt{optimize}(\bmtheta)$
\ENDFOR
\UNTIL{convergence}
\end{algorithmic}
\end{algorithm}

\subsection{Dynamical consensus algorithm}

In the $\mathtt{consensus}$ module, there are two main technical problems. One is how to perform aggregation without a central coordinator. The other is how to eliminate data privacy breach for a given aggregation algorithm. This and the next subsections discuss these problems.

First, let us consider how aggregation is achieved without a central coordinator. Let $\mathsf{A} \in \{0,1\}^{S\times S}$ be the adjacency  matrix of a graph where nodes represent the network participants (or their systems) and edges represent (potentially bilateral) communication channels between the nodes. We assume that each participant knows their connected neighbors. Since the summation over tensors (vectors, matrices, etc.) can be performed element-wise, without loss of generality, we consider the problem of computing the average of scalars $\{ \xi^s \}$:
\begin{align}\label{eq:averageConsensus}
\bar{\xi} = \frac{1}{S}\sum_{s=1}^S \xi^s.
\end{align}
For each $s$, the dynamical consensus algorithm generates a sequence $\xi^s(t)$ with $t=0, 1, 2, \ldots$ and $\xi^s(0) = \xi^s$ such that $\xi^s(\infty)$ converges to $\Bar{\xi}^s$. Upon going from $t$ to $t+1$, the participant receives the current value from the neighboring nodes and performs an update:
\begin{align}\label{eq:Dynamics_elementwise}
\xi^s(t+1) = \xi^s(t) + \frac{1}{S}\sum_{j=1}^S \mathsf{A}_{s,j}[\xi^j(t) - \xi^s(t)]. 
\end{align}
In vector form, this can be written as
\begin{align}\label{eq:Dynamics_matrix}
\bm{\xi}(t+1) = \left[\mathsf{I}_S -(1/S)\mathsf{L}\right]\bm{\xi}(t),
\end{align}
where $\bm{\xi}(t) \triangleq  (\xi^1(t),\ldots,\xi^S(t))^\top$, $\mathsf{I}_S$ is the $S$-dimensional identity matrix, and $\mathsf{L} \triangleq  \mathsf{D}-\mathsf{A}$ is the graph Laplacian. Here $\mathsf{D}$ is the degree matrix defined by $\mathsf{D}_{i,j}=\delta_{i,j}\sum_{s=1}^S \mathsf{A}_{i,s}$ with $\delta_{i,j}$ being Kronecker's delta. 
As can be seen, $\frac{1}{\sqrt{S}}\bm{1}_S$ is an $\ell_2$-normalized eigenvector of $\mathsf{W} \triangleq  \mathsf{I}_S -(1/S)\mathsf{L}$ with an eigenvalue of 1. If $\sfA$ has been properly chosen such that the eigenvalue 1 is unique and the other absolute eigenvalues are less than 1, Eq.~\eqref{eq:Dynamics_matrix} will converge to the stationary solution $\bm{\xi}^*$
\begin{align}
\bm{\xi}^* = \left(1\cdot \frac{1}{\sqrt{S}}\bm{1}_S\frac{1}{\sqrt{S}}\bm{1}_S^\top\right)\bm{\xi}(0)
=  \bm{1}_S \frac{1}{S}\sum_{s=1}^S \xi^s = \bar{\xi}\bm{1}_S.
\end{align}
This implies that \textit{all participants have the same value of $\bar{\xi}$ upon convergence}, achieving average consensus. The distribution of eigenvalues of $\sfA$ greatly affects the convergence speed. Id\'e and Raymond~\cite{ide2021decentralized} demonstrated that the number of iterations until convergence scales with $\ln S$ if the network satisfies the definition of an expander graph. This is a remarkably fast convergence rate. See Fig.~\ref{fig:CbMnetwork.pdf} (right) for an example of an expander graph.

\subsection{Random chunking for data privacy protection}

Next, let us consider how data privacy is protected in the dynamical consensus algorithm. One obvious issue with the above algorithm is that the connected peers can see the original value in the first iteration, which is a breach of data privacy. There are two known solutions to address this issue~\cite{ide2021decentralized}. One is Shamir's secret sharing and the other is random chunking. Here, we will outline the latter. For details of Shamir's secret sharing, please refer to Reference~\cite{ide2021decentralized}. 

The idea behind random chunking is simple. Before running the consensus algorithm, the participants divide the data $\xi^s$ into $C$ chunks $\xi^{s[1]}, \ldots, \xi^{s[C]}$ in any arbitrary manner, as long as $\sum_{l=1}^C \xi^{s[l]} = \xi^s$. Then, they independently run the consensus algorithm $C$ times, each time for a different chunk, to obtain $C$ averages $\bar{\xi}^{[1]}, \ldots, \bar{\xi}^{[C]}$. Since aggregation is a linear operation, we have:
\begin{align}
    \Bar{\xi} = \bar{\xi}^{[1]} + \cdots + \bar{\xi}^{[C]}.
\end{align}

For better privacy protection, it is advisable to randomly change the node labels for each aggregation session so that each participant has a different set of connected peers every time.

\subsection{Sparse Gaussian graphical model mixture: Model definition}
\label{subsec:GGM-mixture_model}

As a concrete example, we provide an algorithm for a sparse Gaussian graphical model (GGM) mixture, which has been extensively discussed in our previous work~\cite{ide2019efficient}. In this model, the data observation model of system $s$ is given by
\begin{align}\label{eq:obs-x}
p({\bm x}^s \mid {\bm z}^s, {\bm \mu}, {\sf \Lambda}) = 
\prod_{k=1}^K {\cal N}({\bm x}^s \mid {\bm \mu}_k, ({\sf \Lambda}_k)^{-1})^{z_k^s}.
\end{align}
Here, $\calN(\cdot \mid \bmmu_k, (\sfLambda_k)^{-1})$ denotes a Gaussian distribution with the mean vector $\bmmu_k \in \mathbb{R}^M$ and the precision matrix $\sfLambda_k  \in \mathbb{R}^{M\times M}$. On the l.h.s., we used the collective notations ${\bm \mu}$ and ${\sf \Lambda}$  representing $\{ {\bm \mu}_k \}$ and $\{ {\sf \Lambda}_k \}$, respectively. In this model, one ``pattern'' is represented by a Gaussian distribution with $(\bmmu_k, \sfLambda_k)$. Since it is unobserved which pattern a sample $\bmx^s$ belongs to, we need a latent variable representing pattern selection. ${\bm z}^s$ is the indicator variable over $K$ different patterns with $z_k^s \in \{ 0,1\}$ for all $s$ and $\sum_{k=1}^K z_k^s=1$. The total number of patterns, $K$, must be provided upfront as a hyper-parameter. However, one can automatically determine an appropriate number starting from a sufficiently large $K$ by using the cardinality regularization technique. See Refs.~\cite{Ide17ICDM,phan2019L0} for technical details. 

For probabilistic inference, prior distributions are assigned to $\bmmu_k,\sfLambda_k, \bmz^s$:
\begin{gather}\label{eq:Gauss-Laplace}
p({\bm \mu}_k,{\sf \Lambda}_k) \propto
{\cal N}({\bm \mu}_k \mid {\bm m}_0,(\lambda_0{\sf \Lambda}_k)^{-1})
\exp\left( -\frac{\rho}{2} \| {\sf \Lambda}_k \|_1 \right)
\\ \label{eq:p(z|pi)}
p({\bm z}^s| {\bm \pi}^s) 
 = \mathrm{Cat}({\bm z}^s| {\bm \pi}^s) 
\triangleq \prod_{k=1}^K (\pi_k^s)^{z^s_k} ,
\end{gather}
where $\rho, \lambda_0,\bm{m}_0$ are predefined constants, ``$\mathrm{Cat}$'' denotes the categorical distribution, and $\|\cdot\|_1$ denotes the $\ell_1$-norm. The parameter $\bm{\pi}^s$ must satisfy $\sum_{k=1}^K \pi_k^s=1$ and $0 \leq \pi_k^s \leq 1$. In this model, task-relatedness is represented by the $K$ sets of model parameters $\{(\bmmu_k,\sfLambda_k)\}_{k=1}^K$ that are shared by all the participants, while the distribution $\bmpi^s$ represents a unique property of the participant $s$. As above, we will use $p(\cdot)$ and $q(\cdot)$ as symbolic notation representing probability density functions in general, rather than a specific functional form.

\subsection{CollabDict implementation of GGM mixture}
\label{subsec:CollaboDict_derivation_GGM_mixture}

In this generative model, the local parameter set $\bmphi^s$ in Algorithm~\ref{algo:CoDiBlock} is 
\begin{align}
    \bmphi^s = \left\{ (\pi^s_1,\ldots, \pi^s_K), (r^{s(1)}_1, \ldots, r^{s(1)}_K), \ldots, (r^{s(N^s)}_1, \ldots, r^{s(N^s)}_K) \right\},
\end{align}
and the intermediate parameters are
\begin{align}
    \bmtheta^s =\left\{ (N^s_1,\ldots,N^s_K), (\bmm^s_1,\ldots, \bmm^s_K), (\sfC^s_1,\ldots, \sfC^s_K)  \right\}.
\end{align}
Here, the intuition of these parameters is as follows. $r^{s(n)}_k$ is the probability that $\bmx^{s(n)}$ belongs to the $k$-th pattern. $N^s_k$ is the average number of samples in $\calD^s$ that fall into the $k$-th pattern. $\pi^s_k$ is the probability that an arbitrary sample in $\calD^s$ belongs to the $k$-th pattern. $\bmm^s_k, \sfC^s_k$ correspond to the mean vector and the scatter matrix of the $k$-th pattern computed only with the samples in $\calD^s$.

$\mathtt{local\_update}$ computes these parameters in the following way, assuming that $\bmmu,\sfLambda$ and $\pi_k^s$ have been properly initialized: 
\begin{gather}\label{eq:r_sk_posterior}
r_k^{s(n)} = \frac{\pi_k^s \calN({\bm x}^{s(n)} \mid {\bm m}_k, ({\sf \Lambda}_{k})^{-1})}{
\sum_{l=1}^K \pi_l^s \calN({\bm x}^{s(n)} \mid {\bm m}_l, ({\sf \Lambda}_{l})^{-1})
}, \\
\label{eq:N_sk}
N^s_k  =\sum_{n \in \calD^s} r_k^{s(n)}, \quad 
\pi^s_k = \frac{N^s_k}{\sum_{l=1}^K N^s_l},\quad  
\bmm^s_k  = \sum_{n \in \calD^s}r_k^{s(n)}\bm{x}^{s(n)}, \quad
\sfC^s_k  = \sum_{n \in \calD^s}r_k^{s(n)}\bm{x}^{s(n)}{\bm{x}^{s(n)}}^\top,
\end{gather}
For each participant $s$, these are computed for all $k=1,\ldots,K$ and $n\in \calD^s$.

Once $\bmphi^s,\bmtheta^s$ are computed for each $s$, the participants start running the $\mathtt{consensus}$ routine to compute aggregations. Specifically, it computes
\begin{align}\label{eq:barN_barm_barC}
    \Bar{N}_k = \sum_{s=1}^S N_k^s,\quad 
    \Bar{\bmm}_k = \frac{1}{\Bar{N}_k}\sum_{s=1}^S \bmm^s_k, \quad 
    \Bar{\sfC}_k = \frac{1}{\Bar{N}_k}\sum_{s=1}^S \sfC^s_k
\end{align}
with the dynamical consensus algorithm. Upon convergence, all the participants end up having these aggregated values in addition to the local parameters $\bmphi^s$.

Finally, each participant runs $\mathtt{local\_update}$ to obtain $\bmtheta=\left\{ (\bmmu_1,\ldots,\bmmu_K), (\sfLambda_1, \ldots, \sfLambda_K)\right\}$ based on $\Bar{N}_k, \Bar{\bmm}_k$ and $\Bar{\sfC}_k$:
\begin{align}
\label{eq:mu-posterior}
\bmmu_k &= \frac{1}{\lambda_0 + \Bar{N}_k}(\lambda_0{\bm m}_0 + \Bar{N}_k\bar{\bmm}_k),
\\ \nonumber
\sfSigma_k &= \bar{\sfC}_k - \bar{\bmm}_k\bar{\bmm}_k^\top +
\frac{\lambda_0}{\lambda_0 + \bar{N}_k}
(\bar{\bmm}_k - \bmm_0)(\bar{\bmm}_k - \bmm_0)^\top,\\
\label{eq:Lamk}
{\sf \Lambda}_{k} &= \arg \max_{{\sf \Lambda}_k}\left\{
 \frac{\bar{N}_k+1}{\bar{N}_k} \ln|{\sf \Lambda}_k|
- \mathrm{Tr}({\sf \Lambda}_k \sfSigma_k) -\frac{\rho}{\bar{N}_k}\|{\sf \Lambda}_k \|_1 \right\},
\end{align}
where $|{\sf \Lambda}_k|$ is the matrix determinant. The objective function in Eq.~\eqref{eq:Lamk} is convex and can be efficiently solved by the graphical lasso algorithm~\cite{Friedman08glasso}. Due to the $\ell_1$ regularization, the resulting precision matrix is sparse, i.e.,~has many zero entries. Sparsity not only improves numerical stability but also contributes to better interpretability. Since the matrix elements of $\sfLambda_k$ represent the partial correlation coefficients according to the theory of the Gaussian graphical model~\cite{Lauritzen1996}, one can easily find directly correlated variables for each variable.

\subsection{Anomaly score of GGM mixture}
\label{subsec:anomaly_score_Gaussian_mixture}

Now that we have described how to obtain model parameters $\{ \bmpi^s, \bmmu, \sfLambda\}$, let us provide a concrete expression for the anomaly score in Eq.~\eqref{ed:anomaly_score_general}. One issue here is that the observation model~\eqref{eq:obs-x} has a latent variable $\bmz^s$. One standard procedure to define an anomaly score in such a case is to use the posterior average. In the present multi-task context, it is given by
\begin{align}\label{eq:anomaly_score_with_latent}
    a^s(\bmx^{s,\mathrm{test}}) = \int \rmd\bmz^s \ q(\bmz^s)\left\{ 
    - \ln p(\bmx^s \mid \bmz^s, \bmmu,\sfLambda)
    \right\}.
\end{align}
Here, $q(\bmz^s)$ is the posterior distribution, given the test sample. Similar to the in-sample posterior in Eq.~\eqref{eq:r_sk_posterior}, $q(\bmz^s)$ is a categorical distribution as
\begin{align}
    q(\bmz^s) = \prod_{k=1}^K \left\{ r^s_k(\bmx^{s,\mathrm{test}})\right\}^{z^s_k}, \quad \quad 
    r_k^{s}(\bmx^{s,\mathrm{test}}) = \frac{\pi_k^s \calN({\bm x}^{s,\mathrm{test}} \mid {\bm m}_k, ({\sf \Lambda}_{k})^{-1})}{
\sum_{l=1}^K \pi_l^s \calN({\bm x}^{s,\mathrm{test}} \mid {\bm m}_l, ({\sf \Lambda}_{l})^{-1})}
\end{align}
and hence, the integration in Eq.~\eqref{eq:anomaly_score_with_latent} should be understood as the summation over the $K$ different choices of $\bmz^s$. Performing the summation, we have
\begin{align}\label{eq:anomalyscore_GGM}
    a^s(\bmx^{s,\mathrm{test}}) = - \sum_{k=1}^K r^s_k(\bmx^{s,\mathrm{test}}) \ln \calN(\bmx^{s,\mathrm{test}} \mid \bmmu_k, (\sfLambda_k)^{-1}).
\end{align}
This expression clearly represents the intuition of dictionary learning. For a test sample $\bmx^{s,\mathrm{test}}$, we first evaluate the degree of fit $r^s_k(\bmx^{s,\mathrm{test}})$ for all $k$'s. Some of the $K$ patterns may have dominating weights while others may be negligible. The final anomaly score is a weighted sum over the one-pattern anomaly scores $\{ -\ln \calN(\bmx^{s,\mathrm{test}} \mid \bmmu_k, (\sfLambda_k)^{-1})\}$. It is straightforward to see that this probabilistic definition of the anomaly score is a natural generalization of the ones based on the reconstruction errors. By using the explicit expression of the Gaussian distribution, we have 
\begin{align}
    a^s(\bmx^{s,\mathrm{test}}) =  \sum_{k=1}^K r^s_k(\bmx^{s,\mathrm{test}}) \left\{
    \frac{M}{2}\ln (2\pi) -\frac{1}{2}\ln |\sfLambda_k| + \frac{1}{2}(\bmx^{s,\mathrm{test}} - \bmmu_k)^\top \sfLambda_k (\bmx^{s,\mathrm{test}} - \bmmu_k)
    \right\}.
\end{align}
The third term in the parenthesis is often called the Mahalanobis distance, which has been used as the main anomaly metric in the classical Hotelling's $T^2$ theory~\cite{Anderson03}. If $\bmx^{s,\mathrm{test}}$ has a dominant weight at a certain $k$, the $\bmmu_k$ can be viewed as a denoised version of $\bmx^{s,\mathrm{test}}$. The similarity to the reconstruction error by the squared loss is obvious.

\section{Deep Decentralized Multi-Task Density Estimation}\label{sec:VAE}

While the sparse GGM mixture-based CollabDict framework serves as a robust and interpretable tool for anomaly detection, its reliance on Gaussian assumptions may limit its effectiveness in systems requiring considerations of higher-order correlations. This section discusses an extension of CollabDict to the variational autoencoder (VAE), a standard deep-learning-based density estimation algorithm.

\subsection{Probabilistic model definition}

Observing the GGM-mixture model discussed in the previous section, one interesting question is what happens if we replace the observation model in Eq.~\eqref{eq:obs-x} with the neural Gaussian distribution:
\begin{align}\label{eq:Neural_Gaussian_obs}
    p(\bmx^s \mid \bmz^s, \bmtheta) &= \calN(\bmx^s \mid \bmmu_{\bmtheta}(\bmz^s), \sfSigma_{\bmtheta}(\bmz^s)).
\end{align}
Here, $\bmz^s$ is a latent variable associated with $\bmx^s$. We assume that $\bmz^s$ is in $\mathbb{R}^d$ and the mean and the covariance matrix in Eq.~\eqref{eq:Neural_Gaussian_obs} are represented by multi-layer perceptrons (MLPs):
\begin{align}\label{eq:Neural_Gaussian_MLP}
    \bmmu_{\bmtheta}(\bmz^s) = \MLP_{\bmtheta}(\bmz^s), \quad 
    \sfSigma_{\bmtheta}(\bmz^s) = \MLP_{\bmtheta}(\bmz^s).
\end{align}
Typically, $\sfSigma_{\bmtheta}$ is assumed to be a diagonal matrix such as $\sfSigma_{\bmtheta}= \diag(\bmsigma_{\bmtheta})^2$, where $\bmsigma_{\bmtheta} \triangleq (\sigma_{\bmtheta,1}, \ldots, \sigma_{\bmtheta,M})^\top$. The functional form of the MLPs has a lot of flexibility. One typical choice is (see, e.g.,~\cite{xu2018unsupervised}):
\begin{align}\label{eq:Neural_Gaussian_MLP2}
    \bmg = \ReLU(\Linear(\bmz^s)), \quad \bmmu_{\bmtheta}=\Linear(\bmg), \quad
    \bmsigma_{\bmtheta} = \ln(1 + \exp(\Linear(\bmg))),
\end{align}
where $\Linear(\cdot)$ performs the Affine transform such that $\Linear(\bmz^s) = \sfW \bmz^s + \bmb$ ($\sfW$ is a parameter matrix and $\bmb$ is a parameter vector to be learned) and $\ReLU$ denotes the rectified linear unit (or the hinge loss function). All the functions operate element-wise. The r.h.s.~of the last equation is called the softmax function. Model parameters $\bmtheta$ to be learned in this case are transformation matrices and intercept vectors in the three Affine transforms. 

Since the latent variable $\bmz^s$ is assumed to be in $\mathbb{R}^d$ this time, we assign a Gaussian prior to $\bmz^s$:
\begin{align}
    p(\bmz^s) = \calN(\bmz^s \mid \bmzero, \sfI_d),
\end{align}
where $\sfI_d$ is the $d$-dimensional identity matrix.

\subsection{Multi-task variational autoencoder}

Since $\bmz^s$ is unobservable, the marginalized likelihood is used to learn the model parameters $\bmtheta$. The log marginalized likelihood is defined by
\begin{align}
    L_0(\bmtheta) &= \ln \prod_{s=1}^S\prod_{n\in\calD^s} \int\rmd \bmz^{s(n)}\ 
    p\left(\bmx^{s(n)} \mid \bmmu_{\bmtheta}(\bmz^{s(n)}), \sfSigma_{\bmtheta}(\bmz^{s(n)})\right)p(\bmz^{s(n)}).
\end{align}
While the integration is analytically intractable, generating samples from the Gaussian prior is straightforward and hence, one might want to use Monte Carlo estimation for estimating the integral. However, as discussed by Kingma and Welling~\cite{kingma2013auto}, direct Monte Carlo estimation of the pdf is numerically challenging, since the probability mass is often concentrated in very limited areas and takes almost zero value everywhere else. 

One standard approach to address these issues is to leverage the variational lower bound derived with Jensen's inequality:
\begin{align}
    L_0\geq L & \triangleq
    \sum_{s=1}^S\sum_{n\in\calD^s} \int\rmd \bmz^{s(n)}\ 
    q(\bmz^{s(n)}) \ln \frac{
    p\left(\bmx^{s(n)} \mid \bmz^{s(n)}, \bmtheta\right)p(\bmz^{s(n)})}{
    q(\bmz^{s(n)}) 
    }.
\end{align}
The inequality holds for any distribution $q(\bmz^{s(n)})$. Notice that the integrand is now $\ln p$, which is more tractable than the pdf itself. Similar to the EM iteration for the mixture model, one can use a two-stage learning strategy: 1) Optimize $q(\cdot)$, given $\bmtheta$, and 2) optimize $\bmtheta$, given $q(\cdot)$. Unlike the GGM mixture case, however, the first stage does not lead to an analytic solution due to the nonlinear dependency on $\bmz^s$ in the MLPs. Thus, we introduce another parametric model to approximate the $q(\bmz^s)$ that would provide the tightest bound:
\begin{align}
    q(\bmz^s) &\approx q(\bmz^s \mid \bmx^s, \bmphi^s) = \calN\left(\bmz^s \mid   \bmm_{\bmphi^s}(\bmx^s), \sfH_{\bmphi^s}(\bmx^s)\right),
\end{align}
where the mean and the covariance matrix are assumed to be MLPs parameterized with $\bmphi^s$ in a similar fashion to Eqs.~\eqref{eq:Neural_Gaussian_MLP}-\eqref{eq:Neural_Gaussian_MLP2}. The first stage is now an optimization problem over $\bmphi^1, \ldots, \bmphi^S$, given $\bmtheta$. Specifically, for $s=1,\ldots,S$, we solve $ \bmphi^s = \arg \max_{\bmphi^s} L^s(\bmphi^s,\bmtheta)$, given the latest $\bmtheta$, where
\begin{align}
   L^s(\bmphi^s,\bmtheta) \triangleq \sum_{n\in\calD^s} \int\rmd \bmz^{s(n)}\ 
    q(\bmz^{s(n)} \mid \bmx^{s(n)}, \bmphi^s)\ln \frac{
    p\left(\bmx^{s(n)} \mid \bmz^{s(n)}, \bmtheta\right)p(\bmz^{s(n)})}{
    q(\bmz^{s(n)} \mid \bmx^{s(n)}, \bmphi^s)
    }.
\end{align}
In the second stage, with an updated $\{\bmphi^1,\ldots, \bmphi^S\}$, we solve
\begin{align}
    \bmtheta = \arg \max_{\bmtheta}\sum_{s=1}^S L^s(\bmphi^s,\bmtheta).
\end{align}

This learning strategy closely resembles the variational autoencoder (VAE)~\cite{kingma2013auto} with one major exception that the ``encoder'' distribution $q(\bmz^s \mid \bmx^s, \bmphi^s)$ now depends on $s$, allowing us to capture situations that are specific to the $s$-th system, and hence, to perform multi-task learning. Figure~\ref{fig:Multi-task_VAE} illustrates the overall architecture of the model. Following the VAE literature, the observation model $p(\bmx^s \mid \bmz^s,\bmtheta)$ can be called the decoder distribution.

\begin{figure}[hbt]
    \centering
    \includegraphics[width=0.9\textwidth]{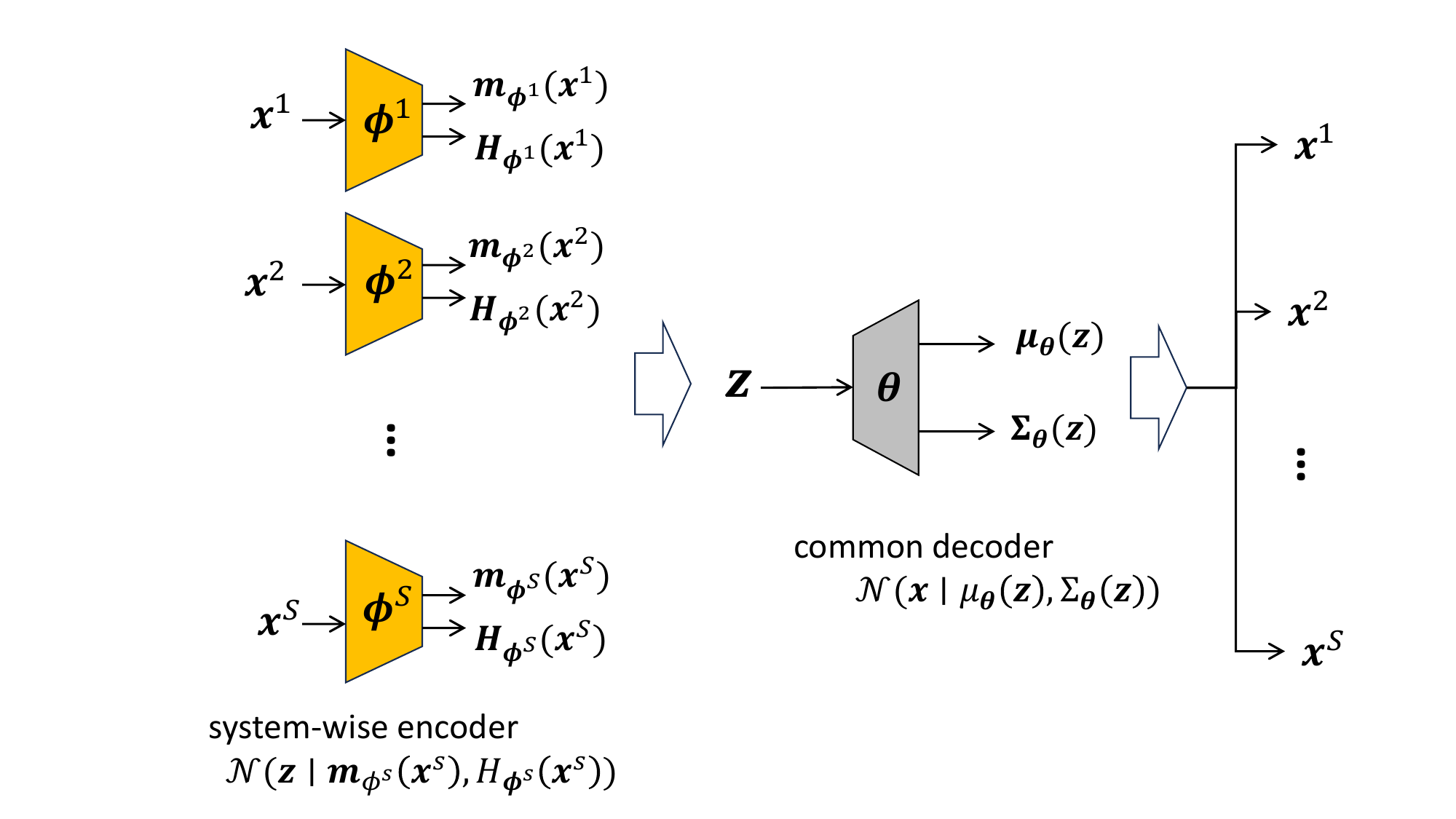}
    \caption{Overall architecture of the proposed multi-task VAE, which operates under decentralized and privacy-preserving constraints.}
    \label{fig:Multi-task_VAE}
\end{figure}

\subsection{Training procedure on CollabDict}

Now let us examine how the above two-stage procedure is applied to the CollabDict framework. First, we observe that the objective function $L^s(\bmphi^s,\bmtheta)$ is expressed as
\begin{align}\nonumber
L^s(\bmphi^s,\bmtheta)
&=\frac{1}{2}\sum_{n\in\calD^s}\left\{
\ln |\sfH_{\bmphi^s}(\bmx^{s(n)})| + d - \trace\left(\sfH_{\bmphi^s}(\bmx^{s(n)})\right) + \|\bmm_{\bmphi^s}(\bmx^{s(n)})\|_2^2
\right\} \\ \label{eq:Objective_L^s}
&+\frac{1}{2}\sum_{n\in\calD^s}\left\langle
-M\ln(2\pi)-\ln | \sfSigma_{\bmtheta}^{s(n)}|
- (\bmx^{s(n)} - \bmmu_{\bmtheta}^{s(n)})^\top (\sfSigma_{\bmtheta}^{s(n)})^{-1}(\bmx^{s(n)} - \bmmu_{\bmtheta}^{s(n)})
\right\rangle_{\bmv},
\end{align}
where $\trace(\cdot)$ represents the matrix trace and $\|\cdot\|_2$ represents the vector $\ell_2$-norm. $\sfH_{\bmphi^s}$ is assumed to be diagonal, such that $\sfH_{\bmphi^s}=\diag(\bmh_{\bmphi^s})^2$ with $\bmh_{\bmphi^s}=(h_{\bmphi^s,1}, \ldots, h_{\bmphi^s, M})^\top$. In the second line, $\langle \cdot \rangle_{\bmv}$ represents the expectation with respect to $\bmv \sim \calN(\bmv \mid \bmzero, \sfI_d)$, and
\begin{align}
\bmmu_{\bmtheta}^{s(n)} \triangleq\bmmu_{\bmtheta}\left(
\bmh_{\bmphi^s}(\bmx^{s(n)})\odot \bmv + \bmm_{\bmphi^s}(\bmx^{s(n)})
\right), \quad
\sfSigma_{\bmtheta}^{s(n)}\triangleq \sfSigma_{\bmtheta}\left(
\bmh_{\bmphi^s}(\bmx^{s(n)})\odot \bmv + \bmm_{\bmphi^s}(\bmx^{s(n)})
\right),
\end{align}
where $\odot$ denotes the element-wise product of vectors. The integration is analytically intractable and has to be done numerically. Fortunately, Monte Carlo sampling from $\calN(\bmv \mid \bmzero, \sfI_d)$ is straightforward. For an arbitrary function $F(\cdot)$, it follows:
\begin{align}
    \langle F(\bmv) \rangle_{\bmv} \approx \frac{1}{J}\sum_{j=1}^J F(\bmv^{[j]}), \quad \mbox{where}\quad \bmv^{[1]}, \ldots, \bmv^{[J]}\sim \calN(\cdot \mid \bmzero, \sfI_d).
\end{align}
This is used to evaluate the second line of $L^s(\bmphi^s,\bmtheta)$ in Eq.~\eqref{eq:Objective_L^s}.

In the proposed multi-task VAE model, stochastic gradient descent (SGD) is employed for parameter estimation. In the $\mathtt{local\_update}$ module of Algorithm~\ref{algo:CoDiBlock}, for each $s$, the local parameter $\bmphi^s$ is updated as
\begin{align}
\bmphi^s \leftarrow \bmphi^s + \eta \frac{\partial L^s(\bmphi^s,\bmtheta)}{\partial \bmphi^s},
\end{align}
where $\bmtheta$ is fixed to its latest numerical value and $\eta$ is the learning rate, which is a hyper-parameter. The mini-batch approach~\cite{zhang2023dive} can be used here, but we have omitted it for simplicity. With the updated $\bmphi^s$, the local gradient with respect to $\bmtheta$ is computed as
\begin{align}
\bmtheta^s = \frac{\partial L^s(\bmphi^s,\bmtheta)}{\partial \bmtheta}.
\end{align}
Note that these calculations can be performed locally using only $\calD^s$.

In the $\mathtt{consensus}$ module, the dynamical consensus algorithm performs simple aggregation:
\begin{align}
    \partial \bmtheta = \mathtt{consensus}(\bmtheta^1, \ldots, \bmtheta^S) = \sum_{s=1}^S \bmtheta^s.
\end{align}

Finally, in the $\mathtt{optimization}$ module, with the latest $\bmphi^s$ and $\partial \bmtheta$, the global parameter $\bmtheta$ is updated:
\begin{align}
    \bmtheta \leftarrow \bmtheta + \eta  \ \partial \bmtheta.
\end{align}
Note that this operation is done by each participant locally in the absence of a central coordinator.

\subsection{Anomaly score}

By running CollabDict for the multi-task VAE model, we can obtain the model parameters $\{\bmphi^s\}, \bmtheta$. We now turn to the general definition of the anomaly score in Eq.~\eqref{eq:anomaly_score_with_latent}. Using the encoder and decoder distributions, we have 
\begin{align}\label{eq:anomaly_score_with_latent_VAE}
    a^s(\bmx^{s,\mathrm{test}}) 
    &= - \int \rmd\bmz^s \ q(\bmz^s\mid  \bmx^{s,\mathrm{test}}, \bmphi^s) 
     \ln p(\bmx^{s,\mathrm{test}} \mid \bmz^s, \bmtheta),\\
     &= - \int \rmd\bmz^s \ \calN\left(\bmz^s \mid   \bmm_{\bmphi^s}(\bmx^{s,\mathrm{test}}), \sfH_{\bmphi^s}(\bmx^{s,\mathrm{test}})\right)
     \ln  \calN(\bmx^{s,\mathrm{test}} \mid \bmmu_{\bmtheta}(\bmz^s), \sfSigma_{\bmtheta}(\bmz^s)).
\end{align}
As discussed in Sec.~\ref{sec:setting}, the logarithmic loss or its expectation for an anomaly score has been used in the literature for a few decades often with a solid information-theoretical background~\cite{lee2000information,Yamanishi2000,staniford2002practical,noto2010anomaly,yamanishi2023learning}. In the particular context of VAE-based anomaly detection, the anomaly score in  Eq.~\eqref{eq:anomaly_score_with_latent_VAE} was first introduced by An and Cho~\cite{an2015variational}, who called it the ``reconstruction probability'' despite the fact that it is not a pdf. Since then, Eq.~\eqref{eq:anomaly_score_with_latent_VAE} has been accepted as one of the standard definitions~\cite{xu2018unsupervised}.

Because of the nonlinearity of MLPs, the integration is analytically intractable. Fortunately, sampling from a multivariate Gaussian with a diagonal covariance matrix is straightforward. Monte Carlo sampling leads to the following expression:
\begin{gather}\label{eq:anomalyscore_VAE}
     a^s(\bmx^{s,\mathrm{test}}) 
    = - \frac{1}{J}\sum_{j=1}^J 
     \ln  \calN(\bmx^{s,\mathrm{test}} \mid \bmmu_{\bmtheta}(\bmz^{[j]}), \sfSigma_{\bmtheta}(\bmz^{[j]})),\\
     \bmz^{[1]}, \ldots, \bmz^{[J]} \sim \calN\left(\cdot \mid   \bmm_{\bmphi^s}(\bmx^{s,\mathrm{test}}), \sfH_{\bmphi^s}(\bmx^{s,\mathrm{test}})\right).
\end{gather}
It is interesting to compare this with the one derived from the GGM mixture model in Eq.~\eqref{eq:anomalyscore_GGM}. In both cases, the anomaly score is a linear combination of the log-loss score of individual patterns, demonstrating the versatility of the CollabDict framework across a wide range of model classes.

\begin{table}[bth]
\centering
\caption{Comparison between decentralized GGM mixture and multitask VAE. }
\label{tab:GGM-VAE_comparison}
\begin{tabular}{l|p{5cm} p{5cm}}
\hline\hline
                         & decentralized GGM mixture & decentralized multi-task VAE \\
\hline
pattern dictionary       & deterministic             & on-demand                    \\[3pt]
interpretability          & 
$r^s_k(\bmx^s)$: sample-wise pattern weight\newline
$\bmpi^s$: system-wise pattern weight\newline
$\bmmu_k$: pattern center\newline
$\sfLambda_k$: variable interdependency                   
& 
$\bmm_{\bmphi^s}(\bmx)$: system-specific embedding\newline
$\sfH_{\bmphi^s}(\bmx)$: embedding confidence\newline
$\bmmu_{\bmtheta}(\bmz)$: sample-wise reconstruction\newline
$\sfSigma_{\bmtheta}(\bmz)$: reconstruction confidence
\\[5pt]
data privacy             & theoretical analysis exist          & (largely open)                       \\[3pt]
model complexity control & well understood             & (largely open)                       \\[3pt]
model instability sources & parameter initialization & parameter initialization\newline
latent dimension $d$\newline
SGD hyperparameters ($\eta$, etc.)\newline
posterior collapse\newline
Monte Carlo sampling\\[3pt]
Computational cost & low & high \\[3pt]
\hline
\end{tabular}
\end{table}

Table~\ref{tab:GGM-VAE_comparison} summarizes the comparison between two model classes for CollabDict. Although deep learning models are generally considered more expressive than non-deep models, such advantages come with a few potential issues. 

In the expression of the anomaly score, the GGM mixture model deterministically stores the pattern dictionary (Eq.~\eqref{eq:anomalyscore_GGM}) while the multi-task VAE generates a pattern dictionary ``on-demand'' through Monte Carlo samples from the posterior distribution (Eq.~\eqref{eq:anomalyscore_VAE}). In the latter, one critical issue is \textit{posterior collapse}~\cite{wang2021posterior,lucas2019understanding,dai2020usual}, which refers to the situation where the posterior $q(\bmz^s\mid  \bmx^{s}, \bmphi^s)$ collapses towards the prior $p(\bmz)$, losing the information of the input sample $\bmx^{s,\mathrm{test}}$. Posterior collapse occurs when the model is overly expressive. When it occurs, the anomaly score can be arbitrarily large even for normal samples, which is a devastating situation in anomaly detection. Due to the randomness of SGD and Monte Carlo sampling, the occurrence of posterior collapse is often unpredictable. Although the issue is reminiscent of rank deficiency in Gaussian mixtures, a few established techniques are available for Gaussian mixtures to properly regularize the model complexity such as $\ell_1$-regularization on $\sfLambda_k$ and $\ell_0$-regularization on $\bmpi^s$~\cite{Ide17ICDM,phan2019L0}. The latter is used as a tool to determine $K$ automatically. 

Another aspect in favor of the GGM mixture is the variety of factors towards model instabilities. In multi-task VAE, tuning hyper-parameters, such as the latent dimension $d$, learning rate $\eta$, batch size, etc., requires a lot of trial and error. This can be problematic in practice as the number of samples, especially anomalous samples, is often limited. Brute-force data-driven parameter tuning strategies, developed in domains that have access to internet-scale datasets (such as speech, text, and images), are often not very useful in real-world applications. This is mainly because the problem setting in real-world scenarios is typically domain-specific, and the benchmark tasks are not directly relevant.

Finally, due to the black-box nature, evaluating the risk of privacy breach is generally challenging in deep learning models. In the next section, we provide some mathematical discussions on data privacy preservation under the GGM mixture model.

\section{Privacy Preservation in CollabDict}\label{sec:privacy}

This section discusses privacy preservation in CollabDict in the GGM mixture setting. Our main focus in this section is to provide a privacy guarantee when the trained model is used by a third party as a ``pre-trained model.'' Additionally, we provide a practical means to monitor internal privacy breaches.  

\subsection{Differentially privacy in CollabDict}
\label{subsec:Differential_privacy_collabdict}

The learned pattern dictionary is shared not only among the participant nodes of the network but also externally as a pre-learned anomaly detection model, depending on business use-cases. In such a scenario, we are concerned about the risk that any of the raw samples are reverse-engineered from the published model. 

The notion of differential privacy~\cite{dwork2014algorithmic} is useful to quantitatively evaluate the privacy risk. Assume that we have created a new dataset $\tilde{\mathcal{D}}$ by perturbing a single sample, say the ${n'}$-th sample of the $s$-th participant: ${\bm{x}}^{s({n'})} \rightarrow \tilde{\bm{x}}^{s({n'})}$. As illustrated in Fig.~\ref{fig:privacy}, when publishing the global state variable, \textit{e.g.}, $\bm{\mu}_k$, we say that $\bm{\mu}_k$ has differential privacy if the state variable computed on $\tilde{\mathcal{D}}$ is not distinguishable from that computed on $\mathcal{D}$ under a data release mechanism. Here the \textit{mechanism} means an appropriately defined data sanitization method, which is to add random noise in this case. 

Formally, we define differential privacy as follows:
\begin{definition}[R\'enyi differential privacy~\cite{mironov2017renyi}]
A mechanism $f: \mathcal{D} \rightarrow \bm{\eta}$ is said to have $(1,\epsilon)$-R\'enyi differential privacy if there exists a constant $\epsilon$ for any adjacent datasets $\mathcal{D}, \tilde{\mathcal{D}}$ satisfying 
\begin{align}\label{eq:Renyi-KL}
\mathrm{KL}[f \mid \mathcal{D},\tilde{\mathcal{D}}] \triangleq 
\int\!\mathrm{d}\bm{\eta}\; f(\bm{\eta}\mid \mathcal{D})
\ln \frac{f(\bm{\eta} \mid \mathcal{D})}{f(\bm{\eta} \mid \tilde{\mathcal{D}})}
 \leq \epsilon.
\end{align}
\end{definition}
In words, the mechanism is differentially private if the adjacent datasets are not distinguishable in terms of the Kullback-Leibler (KL) divergence.

This definition is a special case of the more general $(\alpha,\epsilon)$-R\'enyi differential privacy~\cite{mironov2017renyi}, which also includes the original definition of $\epsilon$-differential privacy~\cite{dwork2014algorithmic} as $\alpha \rightarrow \infty$. A major motivation for this extension is that the $\epsilon$-differential privacy cannot properly handle the Gaussian mechanism and thus is not generally appropriate for real-valued noisy data.

In our Bayesian learning framework, the pattern center parameters $\{\bm{\mu}_k\}$ are potentially the most vulnerable quantity because their posterior mean is represented as a linear combination of the raw samples. One reasonable choice for the data release mechanism is the posterior distribution. While we MAP-estimated $\bmmu_k$ in Sec.~\ref{subsec:CollaboDict_derivation_GGM_mixture}, it is also possible to find the posterior distribution, denoted by $q(\bmmu_k \mid \sfLambda_k)$. As derived in Appendix~\ref{sec:mu_k_posterior}, we have
\begin{align}\label{eq:mu_posterior_with_w}
   q(\bmmu_k \mid \sfLambda_k)=  \mathcal{N}(\bm{\mu}_k \mid \bmw_k,(\lambda_k\mathsf{\Lambda}_k)^{-1}), \quad 
   \bmw_k  \triangleq \frac{1}{\lambda_0 + \Bar{N}_k}(\lambda_0{\bm m}_0 + \Bar{N}_k\bar{\bmm}_k), \quad \lambda_k \triangleq \lambda_0 + \Bar{N}_k.
\end{align}
Notice that the posterior mean $\bmw_k$ is the same as the MAP estimate in Eq.~\eqref{eq:mu-posterior}. Let $\tilde{\bmw}_k$ be the corresponding posterior mean in the perturbed dataset $\Tilde{\calD}$. Assuming the other parameters $\{ \lambda_k, \mathsf{\Lambda}_k \}$ are fixed, we have 
\begin{align} \nonumber
\int\!\mathrm{d}\bm{\eta}\;& \mathcal{N}(\bm{\eta}\mid \bmw_k,(\lambda_k\mathsf{\Lambda}_k)^{-1})
\ln \frac{\mathcal{N}(\bm{\eta}\mid \bm{w}_k,(\lambda_k\mathsf{\Lambda}_k)^{-1})}{\mathcal{N}(\bm{\eta}\mid \tilde{\bm{w}}_k,(\lambda_k\mathsf{\Lambda}_k)^{-1})}
 \\ \nonumber 
 &=\frac{1}{2}\lambda_k (\bm{w}_k - \tilde{\bm{w}}_k)^\top
 \mathsf{\Lambda}_k
 (\bm{w}_k - \tilde{\bm{w}}_k),
 \\ \label{eq:KL_k}
&=\frac{1}{2}\lambda_k \left(
\frac{r^{s(\tilde{n})}_k}{\lambda_k}\right)^2
(\bm{x}^{s(\tilde{n})} - \tilde{\bm{x}}^{s(\tilde{n})})^\top \mathsf{\Lambda}_k
(\bm{x}^{s(\tilde{n})} - \tilde{\bm{x}}^{s(\tilde{n})})
\\ \label{eq:KL_k_bound1}
&\leq 
\frac{1}{2\lambda_k}\max_{j}\{r^{s(j)}_k\}^2 \times
\|\mathsf{\Lambda}_k \|_2 \times
\| \bm{x}^{s(\tilde{n})} - \tilde{\bm{x}}^{s(\tilde{n})}\|^2_2
\end{align}
where $\| \cdot \|_2$ denotes the $\ell_2$ norm. The matrix $\ell_2$ norm is the same as the maximum singular value of the matrix.

Here let us assume that the pairwise distance between samples in the dataset is upper-bounded by $R <\infty$. This can be easily done by removing obvious outliers in the preprocess. As discussed previously, our Bayesian learning algorithm allows for automatically removing irrelevant mixture components during the iteration. Hence, we can further assume that we only retain sufficiently large $\bar{N}_k$ after the \texttt{local\_update} procedure in Algorithm \ref{algo:CoDiBlock},~\textit{i.e.}, $\bar{N}_k > \delta > 0$.

\begin{theorem}
    \label{prop:differentialPrivacy}
We assume that any pairwise $\ell_2$ distance between the samples is upper-bounded by $R < \infty$, and the components whose $\bar{N}_k < \delta$ are discarded after the \texttt{local\_update} procedure in Algorithm~\ref{algo:CoDiBlock}, and $\lambda_0 >0$. Then the release mechanism using the posterior distribution has $(1,\epsilon)$-R\'enyi differential privacy for some constant $\epsilon$.
\end{theorem}

\begin{proof}
We need to bound the right hand-side of Eq.~\eqref{eq:KL_k_bound1} by a constant. From the definition of $\lambda_k$, it is plain to see that $1/\lambda_k < 1/\lambda_0$. Note that $r_k^{s(j)} < 1$ for all $s$ and $j$. It implies that
\begin{equation}
\frac{1}{2\lambda_k}\max_{j}\{r^{s(j)}_k\}^2 \|
\bm{x}^{s(\tilde{n})} - \tilde{\bm{x}}^{s(\tilde{n})}\|^2_2 \le
\frac{R^2}{2\lambda_0}. \label{eq:boundR}
\end{equation}

As shown in Appendix~\ref{sec:upperbound}, there exists an upper-bound in $\|{\sf \Lambda}_{k}\|_2$. Letting $B$ the upper-bound, we have 
\begin{align}
\mathrm{KL}[f \mid \mathcal{D},\tilde{\mathcal{D}}]  \leq
\frac{KBR^2}{2\lambda_0} \triangleq \epsilon,
\end{align}
which completes the proof.
\end{proof}

This theorem guarantees the differential privacy of the CollabDict protocol when publishing the global state variables to external parties.

\begin{figure}[tb]
\begin{center}
\includegraphics[trim={2.5cm 4cm 1cm 4cm},clip,width=12cm]{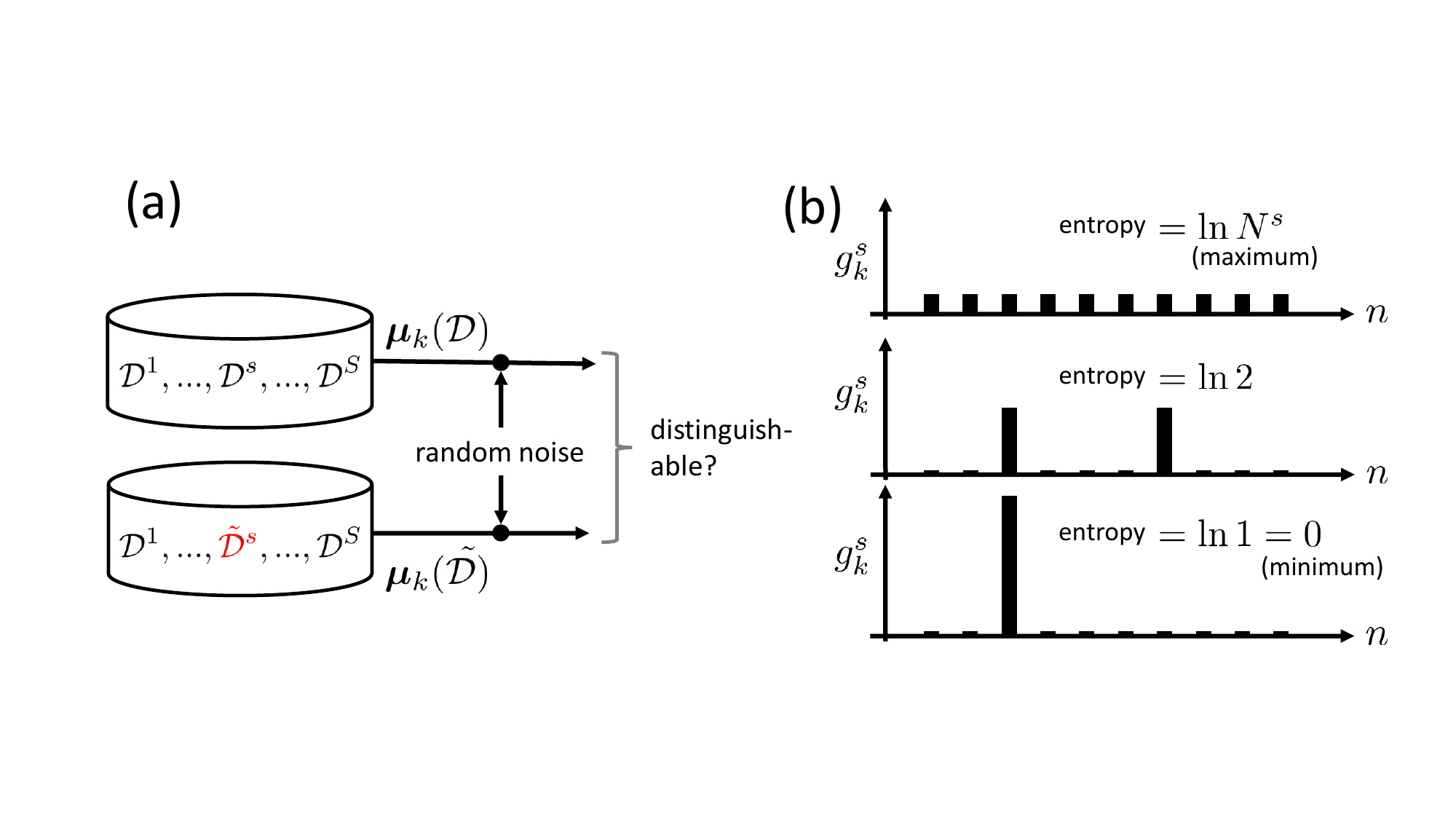}
\end{center}
\caption{Illustration of (a) differential privacy and (b) the entropy $\ell$-diversity. In the latter, the uniform distribution gives the maximum entropy.  }
\label{fig:privacy}
\end{figure}

\subsection{Entropy $\ell$-diversity in CollabDict}

Another aspect of privacy that we are concerned with is privacy preservation in consensus building. In this case, the most vulnerable quantity is potentially $\bmm^s_k$ because it is a linear combination of the raw samples (see Eq.~\eqref{eq:N_sk}). The question is whether the data release mechanism for $\bmm^s_k$ poses any risk of reverse-engineering any of the samples in $\mathcal{D}^s$ through $\bmm^s_k$.

To address this problem, let us define a probability distribution over $n \in \{1,\ldots,N^s\}$:
\begin{align}
g^s_k(n) = \frac{ 
\mathcal{N}(\bm{x}^{s(n)} | \bm{m}_k,(\mathsf{\Lambda}_k)^{-1})
}{
\sum_{m=1}^{N^s}
\mathcal{N}(\bm{x}^{s(m)} | \bm{m}_k,(\mathsf{\Lambda}_k)^{-1})
},
\end{align}
which is a normalized version of the sample weight $r^{s(n)}_k$ in Eq.~\eqref{eq:r_sk_posterior}. If, for example, only a single $n$ dominates the distribution, $\bmm^s_k$ will be a faithful surrogate of the raw sample $\bm{x}^{s(n)}$. Therefore, this sample is exposed to a peer consensus node by sharing $\bmm^s_k$.

The notion of \textit{entropy $\ell$-diversity} is useful  for qualitatively evaluating such a risk:
\begin{definition}[entropy $\ell$-diversity~\cite{Machanavajjhala07ACMtrans}]
 A release mechanism with a probability distribution over database entries is said to have the entropy $\ell$-diversity is the entropy of the distribution is lower-bounded with $\ln \ell$.
\end{definition}
As illustrated in Fig.~\ref{fig:privacy}~(b), entropy is a useful measure of the non-identifiability of the individual samples. The value $\ell$ provides an intuitive measure of how many elements are outstanding in the distribution. 

Therefore, if privacy preservation among the consensus nodes is a significant concern, the participants should monitor the value
\begin{align}
E^s \triangleq  \max_{k} \left\{
- \sum_{n=1}^{N^s} g^s_k(n) \ln g^s_k(n)
\right\}
\end{align}
to determine if it is greater than a threshold $\ln l_0$. If $E^s < \ln l_0$ holds during consensus building, the participant should consider adding random noise to the samples. To avoid introducing unwanted bias to the model, the participant can use the posterior $\mathcal{N}(\cdot \mid \bmw_k, (\lambda_k\mathsf{\Lambda}_k)^{-1})$ for the noise, ensuring privacy preservation as guaranteed by Theorem~\ref{prop:differentialPrivacy}.

\section{Conclusion}

In this paper, we have presented two methodological advancements in decentralized multi-task learning under privacy constraints, aiming to pave the way for future developments in next-generation Blockchain platforms. First, we extended the collaborative dictionary learning (CollabDict) framework, which has been limited to GGM mixture models so far, to incorporate variational autoencoders (VAEs) for enhanced anomaly detection. While VAEs have been widely applied to anomaly detection problems, to the best of our knowledge, this is the first framework satisfying both decentralization and data privacy constraints. 

Furthermore, we proposed a theoretical framework to analyze the risk of data privacy leakage when models trained with the CollabDict framework are shared. Given the widespread application of ``pre-trained models'' in large-scale deep learning, analyzing external data privacy breach is of significant practical importance. We showed that the CollabDict approach, when applied to GGM mixtures, meets $(1,\epsilon)$-R\'enyi differential privacy criterion. We also proposed a practical metric for monitoring internal privacy breaches during the learning process based on the notion of the entropy $\ell$-diversity. 

One promising avenue for future research involves conducting instability analyses of VAE-based anomaly detection methods. In particular, addressing the issue of posterior collapse in VAEs, despite recent advancements, remains a significant challenge. Further research is needed both theoretically and empirically.


\appendix 
\section{Posterior distribution of mean parameter} \label{sec:mu_k_posterior}

This section provides the derivation of the posterior distribution of $\bmmu_k$ discussed in Sec.~\ref{subsec:Differential_privacy_collabdict}.

In the GGM mixture model, we have assigned prior distributions to $\{(\bmmu_k,\sfLambda_k)\}$ and $\bmz^{s(n)}$, and therefore, we can find the posterior distributions for these model parameters using Bayes' rule. We assume the following form for the posterior distributions:
\begin{align}
    Q(\bmmu, \sfLambda,\sfZ) = \prod_{k=1}^K q(\bmmu_k\mid \sfLambda_k) q(\sfLambda_k) \prod_{s=1}^S\prod_{n\in \calD^s} q(\bmz^{s(n)}),
\end{align}
where $\sfZ$ is a collective notation for $\{\bmz^{s(n)}\}$. For $q(\bmz^{s(n)})$, we have provided the posterior as
\begin{align}\label{eq:posterior_for_Z}
    q(\bmz^{s(n)}) = \prod_{k=1}^K \left(r^{s(n)}_k\right)^{z^{s(n)}_k}
\end{align}
in Sec.~\ref{subsec:CollaboDict_derivation_GGM_mixture}. Let us find the posterior distribution for $\bmmu_k$ under the assumption that $\sfLambda$ is point-estimated. Following the variational Bayes framework~\cite{Bishop}, we seek $q(\bmmu_k\mid \sfLambda_k)$ in the form:
\begin{align}
    q(\bmmu_k\mid \sfLambda_k) \propto \left\langle 
    p(\bmmu,\sfLambda)\prod_{s=1}^S\prod_{n\in \calD^s} p(\bmx^{s(n)}\mid \bmz^{s(n)}, \bmmu,\sfLambda)p(\bmz^{s(n)}\mid \bmpi^s)
    \right\rangle_{\sfZ},
\end{align}
where $\langle \cdot \rangle_{\sfZ}$ denotes the expectation with respect to $\prod_{s,n}q(\bmz^{s(n)})$ in Eq.~\eqref{eq:posterior_for_Z}. By taking the logarithm on both sides and selecting the terms that depend on $\bmmu_k$, we have 
\begin{align}\label{eq:ln_q_is}
    \ln q(\bmmu_k\mid \sfLambda_k)
    &=\rmc. -\frac{1}{2}\trace\left( \sfLambda_k \sfQ_k\right),
\end{align}
where $\rmc.$ is a constant and
\begin{gather}
    \sfQ_k \triangleq \Bar{N}_k\Bar{\sfC}_k -\lambda_k\bmw_k\bmw_k^\top - \lambda_0\bmm_0\bmm_0^\top + \lambda_k(\bmmu_k - \bmw_k)(\bmmu_k - \bmw_k)^\top, \\
    \lambda_k \triangleq \lambda_0 + \Bar{N}_k, \quad 
    \bmw_k  \triangleq \frac{1}{\lambda_0 + \Bar{N}_k}(\lambda_0{\bm m}_0 + \Bar{N}_k\bar{\bmm}_k).
\end{gather}
Here, $\bar{N}_k, \bar{\sfC}_K, \bar{\bmm}_k$ have been defined in Eq.~\eqref{eq:barN_barm_barC}. Since $\ln q(\bmmu_k\mid \sfLambda_k)$ in Eq.~\eqref{eq:ln_q_is} is a quadratic form in $\bmmu_k$, we have
\begin{align}
     q(\bmmu_k\mid \sfLambda_k) = \calN\left(\bmmu_k \mid \bmw_k, (\lambda_k\sfLambda_k)^{-1}\right).
\end{align}
This is the posterior distribution in Eq.~\eqref{eq:mu_posterior_with_w} in Sec.~\ref{subsec:Differential_privacy_collabdict}.

\section{Upper-boundedness of precision matrix norm} \label{sec:upperbound}

This section proves the following theorem, which was used in the proof of Theorem~\ref{prop:differentialPrivacy}:

\begin{theorem}\label{th:upperbound}
The $\ell_2$ norm of $\mathsf{\Lambda}_k$ as the solution of Eq.~\eqref{eq:Lamk} is upper-bounded.
\end{theorem}
\begin{proof}

Let us denote the objective function in Eq.~\eqref{eq:Lamk} as $-f(\sfLambda_k)$:
\begin{align}
    f(\sfLambda_k) \triangleq
-b_k \ln|{\sf \Lambda}_k|
+ \mathrm{Tr}({\sf \Lambda}_k \sfSigma_k) +\frac{\rho}{\bar{N}_k}\|{\sf \Lambda}_k \|_1,
\end{align}
where we have defined $b_k \triangleq  \frac{\bar{N}_k+1}{\bar{N}_k}$. Here, we note that $b_k \ln|{\sf \Lambda}_{k}| \le M b_k \ln \|{\sf \Lambda}_{k}\|_2$ holds since $\|{\sf \Lambda}_{k}\|_2$ is the largest eigenvalue of ${\sf \Lambda}_{k}$, which is positive semi-definite. In Theorem~\ref{prop:differentialPrivacy}, we assumed $^\exists  \delta, \ \bar{N}_k > \delta >0$. Hence,
\begin{align} 
Mb_k 
&= M\left(1 + \frac{1}{\bar{N}_k}\right) < M\left(1 + \frac{1}{\delta}\right)  \triangleq b,
\\
\frac{\rho}{\bar{N}_k} 
&\ge \frac{\rho}{S \max_sN^s_k} \ge \frac{\rho}{S \max_sN^s} \triangleq c.
\end{align}
Since both $\sfLambda_k$ and $\sfSigma_k$ are positive semi-definite, $\mathrm{Tr}(\sfLambda_k\sfSigma_k) \ge 0$ holds.  Then, we have
\begin{align} \nonumber
f(\mathsf{I}_M) \ge f({\sf \Lambda}_{k}) 
& \ge -b_k \ln|{\sf \Lambda}_{k}| + \frac{\rho}{N_k}\|{\sf \Lambda}_{k}\|_1 
\ge  -b \ln \|{\sf \Lambda}_{k}\|_2 + c\|{\sf \Lambda}_{k}\|_1,
\end{align}
where $\mathsf{I}_M$ is the $M$-dimensional identity matrix. Hence,
\begin{equation}
\|{\sf \Lambda}_{k}\|_1 \le \frac{f(\mathsf{I}_M)}{c} + \frac{b}{c} \ln
\|{\sf \Lambda}_{k}\|_2. \label{eq:1norm}
\end{equation}
Similarly, we have
\begin{align} \label{eq:2norm}
\mathrm{Tr}(\sfLambda_k\sfSigma_k) 
&\le f(\sfLambda_k) +b_k \ln|{\sf \Lambda}_{k}|
\le f(\mathsf{I}_M) +b_k \ln|{\sf \Lambda}_{k}|
\le f(\mathsf{I}_M) + b \ln \|{\sf \Lambda}_{k}\|_2.
\end{align}
Define $h = \max_{i,j}\left| \ [\sfSigma_k - \mathsf{I}_M]_{i,j}\right|$, where $[\cdot]_{i,j}$ is the operator selecting the $(i,j)$-element of the matrix. The pairwise bounded assumption implies that $h$ is bounded from above. Since
\begin{align*}
\trace(\sfLambda_k\sfSigma_k) = \trace(\sfLambda_k) + \trace\left(\sfLambda_k(\sfSigma_k - \sfI_M)\right) 
\ge \trace(\sfLambda_k) - h \|\sfLambda_k\|_1
\end{align*}
holds, we deduce that
\begin{equation}
\|\sfLambda_k\|_2 
\le \trace(\sfLambda_k) 
\le \trace(\sfLambda_k\sfSigma_k) + h \|\sfLambda_k\|_1.
\label{eq:2normb}
\end{equation}
Combining Eqs.~\eqref{eq:1norm}, \eqref{eq:2norm}, and~\eqref{eq:2normb}, we have 
\begin{equation}
g(\sfLambda_k)  
\triangleq \|\sfLambda_k\|_2 -
\left(1+\frac{d}{c}\right) b \ln \|\sfLambda_k\|_2 -
\left(1+\frac{d}{c}\right) f(\sfI_M)  \le 0 \label{eq:2normfinal}
\end{equation}
Since $g(\cdot)$ is a strictly convex and increases to infinity as
$t$ tends to infinity,  $\|{\sf \Lambda}_{k}\|_2$ is bounded by
the largest solution of $g(t) = 0$ for $t \ge (1+\frac{h}{c})b$. 
\end{proof}


\begin{thebibliography}{10}

\bibitem{an2015variational}
J.~An and S.~Cho.
Variational autoencoder based anomaly detection using reconstruction probability.
 {\em Special lecture on IE}, 2(1):1--18, 2015.

\bibitem{Anderson03}
T.~W. Anderson.
{\em An Introduction to Multivariate Statistical Analysis}.
Wiley-Interscience, 3rd. edition, 2003.

\bibitem{Bishop}
C.~M. Bishop.
{\em Pattern Recognition and Machine Learning}.
Springer-Verlag, 2006.

\bibitem{Chandola09AnomalySurvey}
V.~Chandola, A.~Banerjee, and V.~Kumar.
Anomaly detection: A survey.
{\em ACM Computing Survey}, 41(3):1--58, 2009.

\bibitem{chapel2014anomaly}
L.~Chapel and C.~Friguet.
Anomaly detection with score functions based on the reconstruction error of the kernel PCA.
In {\em Proceedings of Machine Learning and Knowledge Discovery in Databases: European Conference, ECML PKDD 2014, Nancy, France, September 15-19, 2014.}, pages 227--241. Springer, 2014.

\bibitem{chen2020unsupervised}
T.~Chen, X.~Liu, B.~Xia, W.~Wang, and Y.~Lai.
Unsupervised anomaly detection of industrial robots using sliding-window convolutional variational autoencoder.
{\em IEEE Access}, 8:47072--47081, 2020.

\bibitem{chen2021ds2pm}
Y.~Chen, J.~Li, F.~Wang, K.~Yue, Y.~Li, B.~Xing, L.~Zhang, and L.~Chen.
Ds2pm: A data sharing privacy protection model based on blockchain and federated learning.
{\em IEEE Internet of Things Journal}, 2021.

\bibitem{christidis2016blockchains}
K.~Christidis and M.~Devetsikiotis.
Blockchains and smart contracts for the internet of things.
{\em {IEEE Access}}, 4:2292--2303, 2016.

\bibitem{dai2020usual}
B.~Dai, Z.~Wang, and D.~Wipf.
The usual suspects? reassessing blame for vae posterior collapse.
In {\em Proceedings of the 37th International Conference on International conference on machine learning (ICML 20)}, pages 2313--2322. PMLR, 2020.

\bibitem{dwork2014algorithmic}
C.~Dwork, A.~Roth, et~al.
The algorithmic foundations of differential privacy.
{\em Foundations and Trends in Theoretical Computer Science}, 9(3--4):211--407, 2014.

\bibitem{Friedman08glasso}
J.~Friedman, T.~Hastie, and R.~Tibshirani.
Sparse inverse covariance estimation with the graphical lasso.
{\em Biostatistics}, 9(3):432--441, 2008.

\bibitem{huang2022novel}
C.~Huang, Y.~Chai, Z.~Zhu, B.~Liu, and Q.~Tang.
A novel distributed fault detection approach based on the variational autoencoder model.
{\em ACS omega}, 7(3):2996--3006, 2022.

\bibitem{ide2018collaborative}
T.~Id{\'e}.
Collaborative anomaly detection on blockchain from noisy sensor data.
In {\em Proceedings of the 2018 IEEE International Conference on Data Mining Workshops (ICDMW)}, pages 120--127. IEEE, 2018.

\bibitem{Ide17ICDM}
T.~Id\'e, D.~T. Phan, and J.~Kalagnanam.
Multi-task multi-modal models for collective anomaly detection.
In {\em Proceedings of the 17th IEEE Intl. Conf. on Data Mining (ICDM 17)}, pages 177--186, 2017.

\bibitem{ide2021decentralized}
T.~Id{\'e} and R.~Raymond.
Decentralized collaborative learning with probabilistic data protection.
In {\em 2021 IEEE International Conference on Smart Data Services (SMDS)}, pages 234--243. IEEE, 2021.

\bibitem{ide2019efficient}
T.~Id{\'e}, R.~Raymond, and D.~T. Phan.
Efficient protocol for collaborative dictionary learning in decentralized networks.
In {\em Proceedings of the 28th International Joint Conference on Artificial Intelligence}, pages 2585--2591, 2019.

\bibitem{kingma2013auto}
D.~P. Kingma and M.~Welling.
Auto-encoding variational bayes.
{\em arXiv preprint arXiv:1312.6114}, 2013.

\bibitem{Lauritzen1996}
S.~L. Lauritzen.
{\em Graphical Models}.
Oxford, 1996.

\bibitem{lee2000information}
W.~Lee and D.~Xiang.
Information-theoretic measures for anomaly detection.
In {\em Proceedings 2001 IEEE Symposium on Security and Privacy. S\&P 2001}, pages 130--143, 2000.

\bibitem{li2023distvae}
L.~Li, J.~Xiahou, F.~Lin, and S.~Su.
Distvae: Distributed variational autoencoder for sequential recommendation.
{\em Knowledge-Based Systems}, 264:110313, 2023.

\bibitem{lu2017adaptable}
Q.~Lu and X.~Xu.
Adaptable blockchain-based systems: a case study for product traceability.
{\em IEEE Software}, 34(6):21--27, 2017.

\bibitem{lucas2019understanding}
J.~Lucas, G.~Tucker, R.~Grosse, and M.~Norouzi.
Understanding posterior collapse in generative latent variable models, 2019.

\bibitem{Machanavajjhala07ACMtrans}
A.~Machanavajjhala, D.~Kifer, J.~Gehrke, and M.~Venkitasubramaniam.
L-diversity: Privacy beyond k-anonymity.
{\em ACM Trans. Knowl. Discov. Data}, 1(1), Mar. 2007.

\bibitem{mahmood2022blockchain}
Z.~Mahmood and V.~Jusas.
Blockchain-enabled: Multi-layered security federated learning platform for preserving data privacy.
{\em Electronics}, 11(10):1624, 2022.

\bibitem{masurkar2023using}
A.~S. Masurkar, X.~Sun, and J.~Dai.
Using blockchain for decentralized artificial intelligence with data privacy.
In {\em Proceedings of the 2023 International Conference on Computing, Networking and Communications (ICNC)}, pages 195--201. IEEE, 2023.

\bibitem{mironov2017renyi}
I.~Mironov.
Renyi differential privacy.
In {\em Proceedings of the 30th IEEE Computer Security Foundations Symposium (CSF)}, pages 263--275. IEEE, 2017.

\bibitem{munsing2017blockchains}
E.~M{\"u}nsing, J.~Mather, and S.~Moura.
Blockchains for decentralized optimization of energy resources in microgrid networks.
In {\em Proceedings of the 2017 IEEE Conference on Control Technology and Applications (CCTA)}, pages 2164--2171. IEEE, 2017.

\bibitem{nakamoto2008bitcoin}
S.~Nakamoto.
Bitcoin: A peer-to-peer electronic cash system.
2008.

\bibitem{noto2010anomaly}
K.~Noto, C.~Brodley, and D.~Slonim.
Anomaly detection using an ensemble of feature models.
In {\em Proceedings of the 2010 IEEE International Conference on Data Mining (ICDM 10)}, pages 953--958. IEEE, 2010.

\bibitem{omidshafiei2017deep}
S.~Omidshafiei, J.~Pazis, C.~Amato, J.~P. How, and J.~Vian.
Deep decentralized multi-task multi-agent reinforcement learning under partial observability.
In {\em Proceedings of the 34th International Conference on Machine Learning (ICML 17)}, pages 2681--2690. PMLR, 2017.

\bibitem{phan2019L0}
D.~T. Phan and T.~Id{\'e}.
$\ell_0$-regularized sparsity for probabilistic mixture models.
In {\em Proceedings of the 2019 SIAM International Conference on Data Mining (SDM 19)}, pages 172--180. SIAM, 2019.

\bibitem{polato2021federated}
M.~Polato.
Federated variational autoencoder for collaborative filtering.
In {\em Proceedings of the 2021 International Joint Conference on Neural Networks (IJCNN 21)}, pages 1--8. IEEE, 2021.

\bibitem{ruff2021unifying}
L.~Ruff, J.~R. Kauffmann, R.~A. Vandermeulen, G.~Montavon, W.~Samek, M.~Kloft, T.~G. Dietterich, and K.-R. M{\"u}ller.
A unifying review of deep and shallow anomaly detection.
{\em Proceedings of the IEEE}, 109(5):756--795, 2021.

\bibitem{staniford2002practical}
S.~Staniford, J.~A. Hoagland, and J.~M. McAlerney.
Practical automated detection of stealthy portscans.
{\em Journal of Computer Security}, 10(1-2):105--136, 2002.

\bibitem{toyoda2017novel}
K.~Toyoda, P.~T. Mathiopoulos, I.~Sasase, and T.~Ohtsuki.
A novel blockchain-based product ownership management system (POMS) for anti-counterfeits in the post supply chain.
{\em IEEE Access}, 5:17465--17477, 2017.

\bibitem{tse2017blockchain}
D.~Tse, B.~Zhang, Y.~Yang, C.~Cheng, and H.~Mu.
Blockchain application in food supply information security.
In {\em Proceedings of the 2017 IEEE International Conference on Industrial Engineering and Engineering Management (IEEM).}, pages 1357--1361. IEEE, 2017.

\bibitem{wang2021posterior}
Y.~Wang, D.~Blei, and J.~P. Cunningham.
Posterior collapse and latent variable non-identifiability.
{\em Advances in Neural Information Processing Systems}, 34:5443--5455, 2021.

\bibitem{xie2017privacy}
L.~Xie, I.~M. Baytas, K.~Lin, and J.~Zhou.
Privacy-preserving distributed multi-task learning with asynchronous updates.
In {\em Proceedings of the 23rd ACM SIGKDD International Conference on Knowledge Discovery and Data Mining}, pages 1195--1204. ACM, 2017.

\bibitem{xu2018unsupervised}
H.~Xu, W.~Chen, N.~Zhao, Z.~Li, J.~Bu, Z.~Li, Y.~Liu, Y.~Zhao, D.~Pei, Y.~Feng, et~al.
Unsupervised anomaly detection via variational auto-encoder for seasonal {KPIs} in web applications.
In {\em Proceedings of the 2018 world wide web conference}, pages 187--196, 2018.

\bibitem{yamanishi2023learning}
K.~Yamanishi.
{\em Learning with the Minimum Description Length Principle}.
Springer Nature, 2023.

\bibitem{Yamanishi2000}
K.~Yamanishi, J.~Takeuchi, G.~Williams, and P.~Milne.
On-line unsupervised outlier detection using finite mixtures with discounting learning algorithms.
In {\em Proceedings of the Sixth ACM SIGKDD Intl. Conf. on Knowledge Discovery and Data Mining}, pages 320--324, 2000.

\bibitem{zeng2021decentralized}
S.~Zeng, M.~A. Anwar, T.~T. Doan, A.~Raychowdhury, and J.~Romberg.
A decentralized policy gradient approach to multi-task reinforcement learning.
In {\em Proceedings of the thirty-seventh conference on Uncertainty in Artificial Intelligence (UAI 21)}, pages 1002--1012. PMLR, 2021.

\bibitem{zhang2023dive}
A.~Zhang, Z.~C. Lipton, M.~Li, and A.~J. Smola.
{\em Dive into deep learning}.
Cambridge University Press, 2023.

\bibitem{zhang2021federated}
K.~Zhang, Y.~Jiang, L.~Seversky, C.~Xu, D.~Liu, and H.~Song.
Federated variational learning for anomaly detection in multivariate time series.
In {\em Proceedings of the 2021 IEEE International Performance, Computing, and Communications Conference (IPCCC)}, pages 1--9. IEEE, 2021.

\bibitem{zhao2019multi}
Y.~Zhao, J.~Chen, D.~Wu, J.~Teng, and S.~Yu.
Multi-task network anomaly detection using federated learning.
In {\em Proceedings of the 10th international symposium on information and communication technology}, pages 273--279, 2019.

\bibitem{zhao2024deep}
Z.~Zhao, X.~Liang, H.~Huang, and K.~Wang.
Deep federated learning hybrid optimization model based on encrypted aligned data.
{\em Pattern Recognition}, 148:110193, 2024.

\end{thebibliography}

\begin{small}
\paragraph{Disclaimer}
This paper was prepared for informational purposes by the Global Technology Applied Research center of JPMorgan Chase \& Co. This paper is not a product of the Research Department of JPMorgan Chase \& Co. or its affiliates. Neither JPMorgan Chase \& Co. nor any of its affiliates makes any explicit or implied representation or warranty and none of them accept any liability in connection with this paper, including, without limitation, with respect to the completeness, accuracy, or reliability of the information contained herein and the potential legal, compliance, tax, or accounting effects thereof. This document is not intended as investment research or investment advice, or as a recommendation, offer, or solicitation for the purchase or sale of any security, financial instrument, financial product or service, or to be used in any way for evaluating the merits of participating in any transaction.
\end{small}

\end{document}